\documentclass[11pt,a4paper]{article}
\usepackage[hyperref]{acl2019}
\usepackage{times}
\usepackage{latexsym}
\usepackage{url}

\aclfinalcopy % Uncomment this line for the final submission
\usepackage{multirow,multicol}
\usepackage{amsmath}
\usepackage{tabularx}
\usepackage{makecell}
\usepackage{algorithm}
\usepackage{subfig}
\usepackage[noend]{algpseudocode}

\usepackage{helvet}
\usepackage{courier}
\usepackage{bm}
\usepackage{amsmath}
\usepackage{mdwlist}
\usepackage[noend]{algpseudocode}
\usepackage{relsize}
\usepackage{amsfonts}
\usepackage{graphicx}
\usepackage{caption}
\usepackage{hhline}
\usepackage{color}
\usepackage{colortbl}
\usepackage{booktabs}

\usepackage{amsthm}

\DeclareMathOperator*{\argmin}{arg\,min}

\newenvironment{theorem}[2][Theorem]{\begin{trivlist}
\newtheorem{remark}{Remark}
\item[\hskip \labelsep {\bfseries #1}\hskip \labelsep {\bfseries #2.}]}{\end{trivlist}}

\newcommand{\rbm}[1]{\bm{\mathrm{#1}}}
\makeatletter
\def\BState{\State\hskip-\ALG@thistlm}
\makeatother

\makeatletter

\def\BState{\State\hskip-\ALG@thistlm}
\makeatother

\begin{document}
\title{Distantly Supervised Named Entity Recognition using Positive-Unlabeled Learning}

\author{Minlong Peng\thanks{{ }{ }Equal contribution.}, Xiaoyu Xing$^*$, Qi Zhang, Jinlan Fu, Xuanjing Huang\\
School of Computer Science, Fudan University, Shanghai, China\\
\{mlpeng16,xyxing18,qz,fujl16,xjhuang\}@fudan.edu.cn}

\date{}

\maketitle
\begin{abstract}

%Named entity recognition (NER) is a fundamental component in many natural language processing tasks. Existing supervised methods in this field often require a large amount of labeled data, which is often expensive and time-consuming, to generalize well. 
In this work, we explore the way to perform named entity recognition (NER) using only unlabeled data and named entity dictionaries. To this end, we formulate the task as a positive-unlabeled (PU) learning problem and accordingly propose a novel PU learning algorithm to perform the task. We prove that the proposed algorithm can unbiasedly and consistently estimate the task loss as if there is fully labeled data. A key feature of the proposed method is that it does not require the dictionaries to label every entity within a sentence, and it even does not require the dictionaries to label all of the words constituting an entity. This greatly reduces the requirement on the quality of the dictionaries and makes our method generalize well with quite simple dictionaries. Empirical studies on four public NER datasets demonstrate the effectiveness of our proposed method. We have published the source code at \url{https://github.com/v-mipeng/LexiconNER}.

\end{abstract}

\section{Introduction} \label{sec:introduction}

Named Entity Recognition (NER) is concerned with identifying named entities, such as person, location, product and organization names in unstructured text. It is a fundamental component in many natural language processing tasks such as machine translation \cite{babych2003improving}, knowledge base construction \cite{riedel2013relation,shen2012linden}, automatic question answering \cite{bordes2015large}, search \cite{zhu2005espotter}, etc. In this field, supervised methods, ranging from the typical graph models \cite{zhou2002named,mccallum2000maximum,mccallum2003early,settles2004biomedical} to current popular neural-network-based models \cite{chiu2016named,lample2016neural,gridach2017character,Liu18ner,zhang2018chinese}, have achieved great success. However, these supervised methods often require large scale fine-grained annotations (label every word of a sentence) to generalize well. This makes it hard to apply them to label-few domains, e.g., bio/medical domains \cite{deleger2016overview}. 

\begin{figure}[t!] 
\centering
\includegraphics[width=\columnwidth]{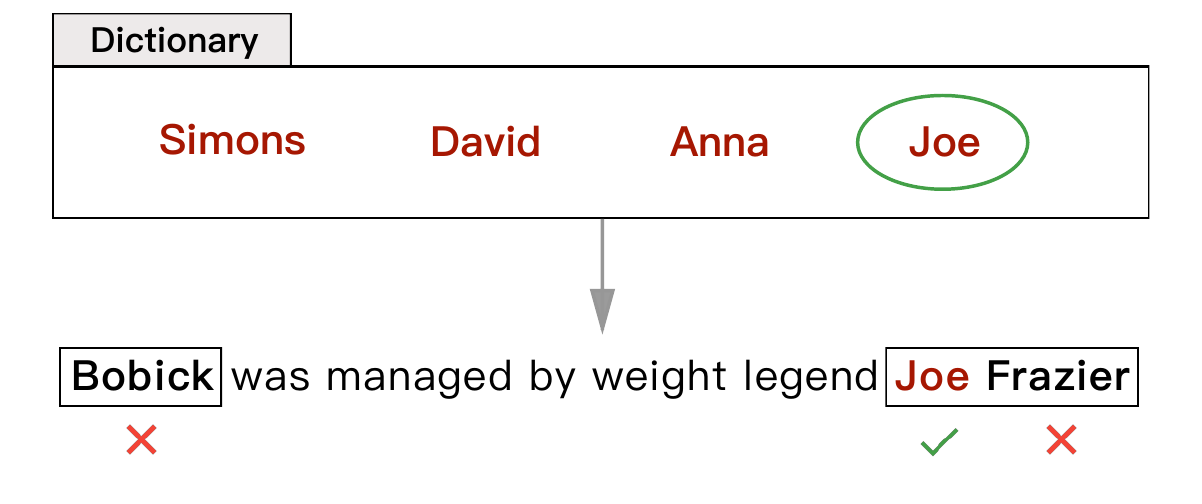}
\caption{Data labeling example for person names using our constructed dictionary. }
\label{fig:label_thumbnail}
\end{figure}

In this work, we explore the way to perform NER using only unlabeled data and named entity dictionaries, which are relatively easier to obtain compared with labeled data. A natural practice to perform the task is to scan through the query text using the dictionary and treat terms matched with a list of entries of the dictionary as the entities \cite{nadeau2006unsupervised,gerner2010linnaeus,liu2015effects,yang2018distantly}. However, this practice requires very high quality named entity dictionaries that cover most of entities, otherwise it will fail with poor performance. As shown in Figure \ref{fig:label_thumbnail}, the constructed dictionary of person names only labels one entity within the query text, which contains two entities ``Bobick" and ``Joe Frazier", and it only labels one word ``Joe" out of the two-word entity ``Joe Frazier". 

To address this problem, an intuitive solution is to further perform supervised or semi-supervised learning using the dictionary labeled data. However, since it does not guarantee that the dictionary covers all entity words (words being of entities) within a sentence, we cannot simply treat a word not labeled by the dictionary as the non-entity word. Take the data labeling results depicted in Figure \ref{fig:label_thumbnail} as an example. Simply treating ``Bobick" and ``Frazier" as non-entity words and then performing supervised learning will introduce label noise to the supervised classifier. Therefore, when using the dictionary to perform data labeling, we can actually only obtain some entity words and a bunch of unlabeled data comprising of both entity and non-entity words. In this case, the conventional supervised or semi-supervised learning algorithms are not suitable, since they usually require labeled data of all classes. 

With this consideration, we propose to formulate the task as a positive-unlabeled (PU) learning problem and accordingly introduce a novel PU learning algorithm to perform the task. In our proposed method, the labeled entity words form the positive (P) data and the rest form the unlabeled (U) data for PU learning. We proved that the proposed algorithm can unbiasedly and consistently estimate the task loss as if there is fully labeled data, under the assumption that the labeled P data can reveal the data distribution of class P. Of course, since words labeled by the dictionary only cover part of entities, it cannot fully reveal data distribution of entity words. To deal with this problem, we propose an adapted method, motivated by the {AdaSampling} algorithm \cite{yang2017positive}, to enrich the dictionary. 
We evaluate the effectiveness of our proposed method on four NER datasets. Experimental results show that it can even achieve comparable performance with several supervised methods, using quite simple dictionaries.

Contributions of this work can be summarized as follows: \textbf{1)} We proposed a novel PU learning algorithm to perform the NER task using only unlabeled data and named entity dictionaries.
\textbf{2)} We proved that the proposed algorithm can unbiasedly and consistently estimate the task loss as if there is fully labeled data, under the assumption that the entities found out by the dictionary can reveal the distribution of entities.
\textbf{3)} To make the above assumption hold as far as possible, we propose an adapted method, motivated by the \textit{AdaSampling} algorithm, to enrich the dictionary.
\textbf{4)} We empirically prove the effectiveness of our proposed method with extensive experimental studies on four NER datasets.

\section{Preliminaries}
 
\subsection{Risk Minimization}
Let $\rbm{X} \in \mathcal{X}$ and $\rm{Y} \in \mathcal{Y}$ be the input and output random variables, where $\mathcal{X} \subset \mathbb{R}^d$ and $\mathcal{Y} = \{0, 1\}$ denote the space of $\rbm{X}$ and $\rm{Y}$, respectively.
Let $f:\mathcal{X} \rightarrow \mathbb{R}$ denote a classifier. A loss function is a map $\ell: \mathbb{R} \times \mathcal{Y} \rightarrow \mathbb{R}^+$. Given any loss function $\ell$ and
a classifier $f$, we define the $\ell$-risk of $f$ by:
\begin{equation} \label{eq:risk}
R_{\ell}(f) = \mathbb{E}_{\rbm{X}, \rm{Y}} \ell(f(\bm{x}), y_{\bm{x}})
\end{equation}
where $\mathbb{E}$ denotes the expectation and its subscript indicates the random variables with respect to which the expectation is taken. 
%When $\ell$ happens to be the $0 - 1$ loss, $R_{\ell}$ refers to our familiar Bayes risk. 
In ordinary supervised learning, we estimate $R_{\ell}$ with the empirical loss $\hat{R}_{\ell}$:
\begin{equation}
    \hat{R}_{\ell} = \frac{1}{n} \sum_{i=1}^n \ell(f(\bm{x}_i), y_i),
\end{equation}
and update model parameters to learn a classifier $f^{*}$ that minimizes $\hat{R}_{\ell}$:
\begin{equation} \label{eq:risk_minization}
	f^* = \argmin_f \hat{R}_{\ell}(f).
\end{equation}

\subsection{Unbiased Positive-Unlabeled learning}
Unbiased positive-unlabeled learning (uPU) \cite{du2014analysis} aims to estimate $R_{\ell}$ when there are only a set of positive (P) examples and a set of unlabeled (U) examples, which contains both positive and negative examples.
$R_{\ell}$ can also be formulated by:
\begin{equation}
    R_\ell = \pi_n \mathbb{E}_{\rbm{X}|\rm{Y}=0} \ell(f(\bm{x}), 0) + \pi_p \mathbb{E}_{\rbm{X}|\rm{Y}=1} \ell(f(\bm{x}), 1),
\end{equation}
where $\pi_p = \rm{P}(\rm{Y}=1)$ and $\pi_n = \rm{P}(\rm{Y}=0)$.
Note that $\mathbb{E}_{\rbm{X}, \rm{Y}=1} \ell(f(\bm{x}), 1)$ can be effectively estimated using positive data. Therefore, the main problem of PU learning is how to estimate $\mathbb{E}_{\rbm{X}|\rm{Y}=0} \ell(f(\bm{x}), 0)$ without using negative labeled data. To this end, it further formulates:
\begin{align*}
    \pi_n \mathbb{E}_{\rbm{X}|\rm{Y}=0} \ell(f(\bm{x}), 0) &= \mathbb{E}_{\rbm{X}} \ell(f(\bm{x}), 0) \\
    & -\pi_p \mathbb{E}_{\rbm{X}|\rm{Y}=1} \ell(f(\bm{x}), 0).
\end{align*}
This equation holds because:
\[\small \rm{P}(\rm{Y}=0) \rm{P}(\rbm{X}|\rm{Y}=0)=\rm{P}(\rbm{X})-P(\rm{Y}=1)\rm{P}(\rbm{X}|\rm{Y}=1).\]
According to this equation, we can now estimate $\pi_n \mathbb{E}_{\rbm{X}| \rm{Y}=0} \ell(f(\bm{x}), 0)$ using only unlabeled data and positive data. Thus, $R_{\ell}$ can be effectively estimated using only unlabeled data and positive data. In summary, we have that $R_{\ell}$ can be unbiasedly estimated by:
\begin{equation} \label{eq:uPU}
\small
\begin{split}
    \hat{R}_{\ell} &= \frac{1}{n_u} \sum_{i=1}^{n_u} \ell(f(\bm{x}_i^u), 0) +\\
    & \frac{\pi_p}{n_p} \sum_{i=1}^{n_p} \left(\ell(f(\bm{x}_i^p), 1) - \ell(f(\bm{x}_i^p), 0)\right),
\end{split}
\end{equation}
where $\bm{x}_i^u$ and $\bm{x}_i^p$ denotes an unlabeled and positive example, respectively, and $n_u$ and $n_p$ denotes the number of unlabeled and positive examples, respectively.

\subsection{Consistent Positive-Unlabeled Learning} \label{sec:consistent_pu}
As we know, a good estimation should be not only unbiased but also consistent. The above induction has proved that $\hat{R}_{\ell}$ is an unbiased estimation of ${R}_{\ell}$. In this section, we show that $\hat{R}_{\ell}$ can be also a consistent estimation of ${R}_{\ell}$ when the loss function $\ell$ is upper bounded. \textit{We argue that this is the first work to give such a proof}, which is summarized in the following theorem:
\begin{theorem}{1} \label{theorem:consistency}
If $\ell$ is bounded by $[0, M]$, then for any $\epsilon > 0$, 
\begin{equation}
\begin{split}
&\rm{P}\{S \in \mathcal{D}|\sup_{f \in \mathcal{H}_R} |{R}_{\ell}-\hat{R}_{\ell}| \leq \epsilon\} \\
&\geq 1 - 2N(\frac{\epsilon}{4(1+2\pi_p)L_M}) e^{-\frac{\min(n_p,  n_u)\epsilon^2}{8(1+2\pi_p)^2B^2}},
\end{split}         
\end{equation}
where $B = L_M M + C_0$.
Here, $L_M$ denotes the Lipschitz constant that $L_M > \frac{\partial \ell(w, y)}{\partial w}, \forall w \in \mathbb{R}$, $C_0 = \max_y \ell(0, y)$, and $\mathcal{H}$ denotes a Reproducing Kernel Hilbert Space (RKHS) \cite{aronszajn1950theory}. $\mathcal{H}_R$ is the hypothesis space for each given $R > 0$ in the ball of radius $R$ in $\mathcal{H}$. $N(\epsilon)$ denotes the covering number of $\mathcal{H}_R$ following Theorem C in \cite{cucker2002mathematical}.
\end{theorem}
\begin{proof}
Proof appears in Appendix \textbf{A}. 
\end{proof}

\begin{remark}
Let us intuitively think about what if $\ell$ is not upper bounded (e.g., the cross entropy loss function). Suppose that there is a positive example $\bm{x}_i^p$ not occurring in the unlabeled data set. Then, its corresponding risk defined in $\hat{R}_{\ell}$ is $V(\bm{x}_i^p) = \frac{\pi_p}{n_p}(\ell(f(\bm{x}_i^p), 1)-\ell(f(\bm{x}_i^p), 0))$. If $\ell$ is not upper bounded, to achieve a small value of $V(\bm{x}_i^p)$, $f$ can heavily overfit $\bm{x}_i^p$ making $\ell(f(\bm{x}_i^p), 0) \rightarrow +\infty$, and in turn $V(\bm{x}_i^p) \rightarrow -\infty$. From this analysis, we can expect that, \textit{when using a unbounded loss function and a flexible classifier, $\hat{R}_{\ell}$ will dramatically decrease to a far below zero value}. 
\end{remark}

Therefore, in this work, \textit{we force $\ell$ to be bounded by replacing the common unbounded cross entropy loss function with the mean absolute error, resulting in a bounded unbiased positive-unlabeled learning (buPU) algorithm}. This slightly differs from the setting of uPU, which only requires $\ell$ to be symmetric.

We further combine buPU with the non-negative constraint proposed by \citet{kiryo2017positive}, which has proved to be effectiveness in alleviating overfitting, obtaining a bounded non-negative positive-unlabeled learning (bnPU) algorithm:
\begin{equation}\label{eq:nnpu}
\small
\begin{split}
\hat{R}_{\ell} &= \frac{\pi_p}{n_p} \sum_{i=1}^{n_p} \ell(f(\bm{x}_i^p), 1) + \\
& \max\left(0, \frac{1}{n_u} \sum_{i=1}^{n_u} \ell(f(\bm{x}_i^u), 0) -  \frac{\pi_p}{n_p} \sum_{i=1}^{n_p} \ell(f(\bm{x}_i^p), 0)\right).
\end{split}
\end{equation}
%with $\ell \in [0, M]$.

\section{Dictionary-based NER with PU Learning}

In the following, we first define some notations used throughout this work, and illustrate the label assignment mechanism used in our method. Then, we precisely illustrate the data labeling process using the dictionary. After that, we show the detail for building the PU classifier, including word representation, loss definition, and label inference. Finally, we show the adapted method for enriching the dictionary.

\subsection{Notations}
We denote $\rm{W} \in \mathcal{V}$ and $\rm{S} = \{\rm{W}\} \in \mathcal{S}$ be the word-level and sentence-level input random variables, where $\mathcal{V}$ is the word vocabulary and $\mathcal{S}$ is the sentence space. $\rm{D}_e$ denotes the entity dictionary for a given entity type and $\mathcal{D}=\{{s}_1, \cdots, {s}_N\} \subseteq \mathcal{S}$ denotes the unlabelled dataset.
We denote $\mathcal{D}^+$ the set of entity words labeled by $\rm{D}_e$, and denote $\mathcal{D}^u$ the rest unlabeled words. 

\subsection{Label Assignment Mechanism}
In this work, we apply the binary label assignment mechanism for the NER task instead of the prevalent BIO or BIOES mechanism. Entity words are mapped to the positive class and non-entity words are mapped to the negative class. This is because, as we have discussed in the \textsection \ref{sec:introduction}, the dictionary cannot guarantee to cover all entity words within a sentence. It may only label the beginning (B), the internal (I), or the last (E) word of an entity. Therefore, we cannot distinguish which type, B, I, or E, the labeled entity word belongs to. Take the data labeling results depicted in Figure \ref{fig:label_thumbnail} as an example. With the dictionary, we know that ``Joe" is an entity word. However we cannot know that it is the beginning of the person name ``Joe Frazier".

\subsection{Data Labeling using the Dictionary}

To obtain $\mathcal{D}^+$, we use the maximum matching algorithm \cite{liu1994word,xue2003chinese} to perform data labeling with $\rm{D}_e$. It is a greedy search routine that walks through a sentence trying to find the longest string, starting from a given point in the sentence, that matches with an entry in the dictionary. The general process of this algorithm is summarized in Alg. \ref{alg:entity_label}. 
In our experiments, we intuitively set the context size $k=4$. 

\begin{algorithm}[t]
\caption{Data Labeing using the Dictionary}\label{alg:entity_label}
\begin{algorithmic}[1]
\BState \textbf{Input:} named entity dictionary $\rm{D}_e$, a sentence ${s}= \{{w}_1, \cdots, {w}_n\}$, and the context size $k$
\BState \textbf{Result:} partial labeled sentence
\BState \textbf{Initialize:} $i \gets 1$
\While {$i\leq n$}
	\For{$j \in [k, \cdots, 0]$} 
    	\If{$\{{w}_i, \cdots, {w}_{\max(i+j, n)}\} \in \rm{D}_e$} 
        	\State label $\{{w}_i, \cdots, {w}_{\max(i+j, n)}\}$ as positive class.
        	\State $i \gets i+j+1$
        	\State break
       	\EndIf
        \If{$j==0$}
        \State $i \gets i+1$
        \EndIf
    \EndFor
\EndWhile
\end{algorithmic}
\end{algorithm}

\subsection{Build PU Learning Classifier}
In this work, we use a neural-network-based architecture to implement the classifier $f$, and this architecture is shared by different entity types.

\paragraph{Word Representation.}
Context-independent word representation consists of three part of features, i.e., the character sequence representation $\bm{e}_c({w})$, the word embedding $\bm{e}_w({w})$, and some human designed features on the word-face $\bm{e}_h({w})$. 

For the character-level representation $\bm{e}_c({w})$ of ${w}$, we use the \textit{one-layer convolution network} model \cite{kim2014convolutional} on its character sequence $\{c_1, c_2, \cdots, c_m\} \in \mathcal{V}_c$, where $\mathcal{V}_c$ is the character vocabulary. Each character $c$ is represented using
\begin{equation*}
   \bm{v}(c) = \rbm{W}_c(c),
\end{equation*}
where $\rbm{W}_c$ denotes a character embedding lookup table. 
The \textit{one-layer convolution network} is then applied to $\{\bm{v}(c_1), \bm{v}(c_2), \cdots, \bm{v}(c_m)\}$ to obtain $\bm{e}_c({w})$.

For the word-level representation $\bm{e}_w({w})$ of ${w}$, we introduce an unique dense vector for $w$, which is initialized with Stanford's GloVe word embeddings\footnote{ http://nlp.stanford.edu/projects/glove/} \cite{pennington2014glove} and fine-tuned during model training. 

For the human designed features $\bm{e}_h({w})$ of ${w}$, we introduce a set of binary feature indicators. These indicators are designed on options proposed by \citet{collobert2011natural}: \textit{allCaps}, \textit{upperInitial}, \textit{lowercase}, \textit{mixedCaps}, \textit{noinfo}. If any feature is activated, its corresponding indicator is set to 1, otherwise 0. This way, it can keep the capitalization information erased during lookup of the word embedding.  

The final word presentation independent to its context $\bm{e}({w}) \in R^{k_w}$ of ${w}$, is obtained by concatenating these three part of features:
\begin{equation}
\bm{e}({w}) = [\bm{e}_c({w}) \oplus\bm{e}_w({w}) \oplus \bm{e}_h({w})],
\end{equation}
where $\oplus$ denotes the concatenation operation. Based on this representation, we apply a bidirectional LSTM (BiLSTM) network \cite{huang2015bidirectional}, taking $\bm{e}({w}_t), {w}_t \in {s}$ as step input, to model context information of ${w}_t$ given the sentence $s$. Hidden states of the forward and backward LSTMs at the $t$ step are concatenated:
\begin{equation}
\bm{e}({{w}_t|{s}})=[\overrightarrow{\rbm{h}}_t \oplus \overleftarrow{\rbm{h}}_t],
\end{equation}
to form the representation of ${w}_t$ given ${s}$.

\paragraph{Loss Definition.}
Given the word representation, $\bm{e}({w}|{s})$, of ${w}$ conditional on ${s}$, its probability to be predicted as positive class is modeled by:
\begin{equation}
f({w}|{s}) = \sigma (\bm{w}_p^T \bm{e}({w}|{s})+b),
\end{equation}
where $\sigma$ denotes the sigmoid function, $\bm{w}_p$ is a trainable parameter vector and $b$ is the bias term. The prediction risk on this word given label $y$ is defined by:
\begin{equation}
\begin{split}
\ell(f({w}|{s}), y) = |y-f({w}|{s})|.
\end{split}
\end{equation} 
Note that $\ell(f({w}|{s}), y) \in [0, 1)$ is upper bounded. The empirical training loss is defined by:
\begin{equation} \label{eq:est_pi}
\resizebox{\columnwidth}{!}{$
\hat{R}_{\ell}(f) = \pi_p \hat{R}_p^+(f) + \max \left\{0, \hat{R}_u^-(f) - {\pi}_p \hat{R}_p^-(f)\right\},$}
\end{equation}
where \small
\begin{align*}
    \hat{R}_p^+(f) & = \tiny{\frac{1}{|\mathcal{D}^+|} \sum_{{w}|{s} \in \mathcal{D}^+} \ell(f({w}|{s}), 1)}, \\
    \hat{R}_p^-(f) &= 1-\hat{R}_p^+(f), \\
    \hat{R}_u^-(f) &= \frac{1}{|\mathcal{D}^u|} \sum_{{w}|{s} \in \mathcal{D}^u } \ell(f({w}|{s}), 0),
\end{align*}
\normalsize
and $\pi_p$ is the ratio of entity words within $\mathcal{D}^u$.
%and we will show detail of its estimation in the ``Experiment" section.

In addition, during our experiments, we find out that due to the class imbalance problem ($\pi_p$ is very small), $f$ inclines to predict all instances as the negative class, achieving a high value of accuracy while a small value of F1 on the positive class. This is unacceptable for NER. Therefore, we introduce a class weight $\gamma$ for the positive class and accordingly redefine the training loss as:
\begin{equation}
\resizebox{\columnwidth}{!}{$
\hat{R}_{\ell}(f) = \gamma \cdot \pi_p \hat{R}_p^+(f) + \max \left\{0, \hat{R}_u^-(f) - \pi_p \hat{R}_p^-(f)\right\}.$}
\end{equation}
%During our experiments, we find out that it performs well when $\gamma \cdot \pi_p \approx 0.5$.

\paragraph{Label Inference.}
Once the PU classifier has been trained, we use it to perform label prediction.
However, since we build a distinct classifier for each entity type, a word may be predicted as positive class by multiple classifiers. To address the conflict, we choose the type with the highest prediction probability (evaluated by $f(w|s)$). Predictions of classifiers of the other types are reset to 0.

At inference time, we first solve the type conflict using the above method. After that, consecutive words being predicted as positive class by the classifier of the same type are treated as an entity. Specifically, for sequence ${s}=\{{w}_1, {w}_2, {w}_3, {w}_4, {w}_5\}$, if its predicted labels by the classifier of a given type are $L=\{1, 1, 0, 0, 1\}$, then we treat $\{{w}_1, {w}_2\}$ and $\{{w}_5\}$ as two entities of the type. 

\subsection{Adapted PU Learning for NER}

In PU learning, we use the empirical risk on labeled positive data,  $\frac{1}{n_p} \sum_{i=1}^{n_p} \ell(f({x}_i^p), 1)$, to estimate the expectation risk of positive data. This requires that the positive examples ${x}_i^p$ draw identically independent from the distribution $\rm{P}(\rm{X}|\rm{Y}=1)$. The requirement is usually hard to satisfy, using a simple dictionary to perform data labeling. 

To alleviate this problem, we propose an adapted method, motivated by the {AdaSampling} \cite{yang2017positive} algorithm. The key idea of the proposed method is to adaptively enrich the named entity dictionary. Specifically, we first train a PU learning classifier $f$ and use it to label the unlabeled dataset. Based on the predicted label, it extracts all of the predicted entities. For a predicted entity, if it occurs over $k$ times and all of its occurrences within the unlabeled dataset are predicted as entities, we will add it into the entity dictionary in the next iteration. This process iterates several times until the dictionary does not change.

\section{Experiments} 
In this section, we empirically study:
\begin{itemize}
\item \textit{the general performance of our proposed method using simple dictionaries;}
\item \textit{the influence of the unlabeled data size;}
\item \textit{the influence of dictionary quality, such as size, data labeling precision and recall;}
\item \textit{and the influence of the estimation of $\pi_p$.}
\end{itemize}

\subsection{Compared Methods}

There are five indispensable baselines with which our proposed Adapted PU learning (\textbf{AdaPU}) algorithm should compare. The first one is the dictionary matching method, which we call \textbf{Matching}. It directly uses the constructed named entity dictionary to label the testing set as illustrated in Alg. \ref{alg:entity_label}. The second one is the supervised method that uses the same architecture as $f$ but trains on fine-grained annotations (fully labeled $\mathcal{D}^u$ and $\mathcal{D}^+$). In addition, it applies the BIOES label assignment mechanism for model training. We call this baseline \textbf{BiLSTM}. The third one is the \textbf{uPU} algorithm, which uses cross entropy loss to implement $\ell$. The fourth one is the bounded uPU (\textbf{buPU}) algorithm, which implement $ell$ with mean absolute error. Compared with AdaPU, it does not apply the non-negative constraint and does not perform dictionary adaptation. The last one is the bounded non-negative PU learning (\textbf{bnPU}) algorithm, which does not perform dictionary adaptation compared with AdaPU. 

Additionally, we compared our method with several representative supervised methods that have achieved state-of-the-art performance on NER. These methods include:
\textbf{Stanford NER (MEMM)} \cite{mccallum2000maximum} a maximum-entropy-markov-model-based method;
\textbf{Stanford NER (CRF)} \cite{finkel2005incorporating} a conditional-random-field-based method;
and \textbf{BiLSTM+CRF} \cite{huang2015bidirectional} a neural-network-based method as the BiLSTM baseline, but additionally introducing a CRF layer.

\begin{table}[t!]
\centering
\renewcommand{\arraystretch}{0.7}
\resizebox{\columnwidth}{!}
{\begin{tabular}{llrrr} 
\toprule
Dataset& Type & \# of l.w. & Precision & Recall \\
	\midrule
	\multirow{4}{*}{CoNLL (en)} &PER  &2,507 & 89.26 & 17.38\\ %\cmidrule(lr){3-6}
	&LOC &4,384 & 85.07 & 50.03\\ %\cmidrule(lr){3-6}
    &ORG &3,198 & 86.17 & 29.45\\ %\cmidrule(lr){3-6}
    &MISC &1,464 & 92.13 &30.59\\ \midrule
    \multirow{4}{*}{CoNLL (sp)} &PER  &574 &90.24 &37.84\\ %\cmidrule(lr){3-6}
	&LOC &272 &84.93 &16.39 \\ %\cmidrule(lr){3-6}
    &ORG &702 &96.87 &27.19 \\ %\cmidrule(lr){3-6}
    &MISC &157 &68.15 &11.94 \\ \midrule
	\multirow{3}{*}{MUC} &PER  &788 & 74.56 & 28.50\\ 
	&LOC  &511 & 89.43 &43.33\\ %\cmidrule(lr){3-6}
    &ORG &1,257 &97.74 &30.38\\
    \midrule
	\multirow{3}{*}{Twitter} &PER &1,842 & 79.26 &26.03\\ 
	&LOC &1,109 &90.96  &34.15\\ %\cmidrule(lr){3-6}
    &ORG &398 &83.77 &20.58\\
\bottomrule
\end{tabular}}
\caption{Data labeling results using the dictionary: the number of labeled words (\# of l.w.), the word-level precision ($\frac{\text{\# of true labeled words}}{\text{\# of total labeled words}}$) and recall.}
\label{table:dict-info}
\end{table}

\subsection{Datasets} 
\paragraph{CoNLL (en).} CoNLL2003 NER Shared Task Dataset in English \cite{tjong2003introduction} collected from Reuters News. It is annotated by four types: PER, LOC, ORG, and MISC. We used the official split training set for model training, and testb for testing in our experiments, which contains 203K and 46K tokens, respectively. In addition, there are about 456k additional unlabeled tokens. 

\paragraph{CoNLL (sp).} CoNLL2002 Spanish NER Shared Task Dataset \cite{Sang2002Introduction} collected from Spanish EFE News Agency. It is also annotated by PER, LOC, ORG, and MISC types. The training and test data sets contain 273k and 53k lines, respectively. 

\paragraph{MUC.} Message Understanding Conference 7 released by \citet{chinchor1998overview} for NER. It has about 190K tokens in the training set and 64K tokens in the testing set. For the sake of homogeneity, we perform entity detection on PER, LOC, and ORG in this study.

\paragraph{Twitter.} Twitter is a dataset collected from Twitter and released by \citet{zhang2018adaptive}. It contains 4,000 tweets for training and 3,257 tweets for testing. Every tweet contains both textual information and visual information. In this work, we only used the textual information to perform NER and we also only performed entity detection on PER, LOC, and ORG.

For the proposed method and the PU-learning-based baselines, we used the training set of each dataset as $\mathcal{D}$. Note that we did not use label information of each training set for training these models.

\begin{table}[t!]
\centering
\resizebox{\columnwidth}{!}{\begin{tabular}{lcccc}
\toprule
Dataset& PER& LOC& ORG& MISC\\
\midrule
   	CoNLL (en)	& .055/.053 & .041/.038& .049/.045& .023/.020 \\
    CoNLL (sp)	& .019/.018 & .019/.017& .030/.027& $----$\\
	MUC-7\,	& .022/.019	& .025/.023& .037/.034& $----$
    \\
    Twitter & .058/.055	& .046/.044 &.021/.018 & $----$ \\
\bottomrule
\end{tabular}}
\caption{True/Estimated value of $\pi_p$.}
\label{table:data_prior}
\vspace*{-\baselineskip}
\end{table}

\begin{table*}[t!] 
\renewcommand{\arraystretch}{0.9}
\centering
\resizebox{\textwidth}{!}{\begin{tabular}{ll|cccc|ccccr}
\hline
Dataset                 		&Type     	& MEMM 	& CRF 	& BiLSTM 	& BiLSTM+CRF 	& Matching 	& uPU 	& buPU 	& bnPU 	& AdaPU \\ \hline
\multirow{5}{*}{CoNLL (en)}  	& PER 		& 91.61 & 93.12 & 94.21 	& 95.71  	   	& 6.70 		& 74.22	& 85.01	& 87.21 	& 90.17  \\
								& LOC  		& 89.72 & 91.15 & 91.76 	& 93.02 	   	& 67.16  	& 69.88	& 81.27	& 83.37 	& 85.62  \\
								& ORG 		& 80.60 & 81.91 & 83.21 	& 88.45 	   	& 46.65 	& 73.64	& 74.72	& 75.29 	& 76.03  \\
								& MISC 		& 77.45 & 79.35 & 76.00 	& 79.86 	   	& 53.98  	& 68.90	& 68.90	& 66.88 	& 69.30  \\
								& Overall	& 86.13 & 87.94 & 88.30 	& 90.01			& 44.90		& 72.32	& 79.20	& 80.74 	& 82.94  \\
\hline
\multirow{5}{*}{CoNLL (sp)}  	& PER 		& 86.18 & 86.77 & 88.93 	& 90.41   		& 32.40 	& 82.28 & 83.76	& 84.30 	& 85.10   \\
								& LOC  	 	& 78.48 & 80.30 & 75.43  	& 80.55 	 	& 28.53  	& 70.44	& 72.55	& 73.68 	& 75.23  \\
								& ORG 		& 79.23 & 80.83 & 79.27 	& 83.26 		& 55.76 	& 69.82	& 71.22	& 69.82 	& 72.28 \\
								& Overall	& 81.14 & 82.63 & 80.28 	& 84.74  		& 42.23 	& 73.84	& 74.50	& 74.43		& 75.85  \\
\hline                                  	                                                        	    		        		
\multirow{4}{*}{MUC}  	 	 	& PER 	 	& 86.32 & 87.50 & 85.71  	& 84.55  	 	& 27.84 	& 77.98	& 84.94	& 84.21 	& 85.26 \\
								& LOC 		& 81.70 & 83.83 & 79.48 	& 83.43 		& 62.82 	& 64.56	& 72.62	& 75.61 	& 77.35 \\
								& ORG 		& 68.48 & 72.33 & 66.17 	& 67.66 		& 51.60 	& 45.30	& 58.39	& 58.75 	& 60.15 \\  
								& Overall	& 74.66 & 76.47 & 73.12 	& 75.08 		& 50.12		& 63.87	& 69.89	& 70.06 	& 71.60  \\ \hline
\multirow{4}{*}{Twitter}  	 	& PER 		& 73.85 & 80.86 & 80.61  	& 80.77  	 	& 41.33  	& 67.30	& 72.72	& 72.68 	& 74.66 \\
								& LOC 		& 69.35 & 75.39 & 73.52 	& 72.56 		& 49.74 	& 59.28	& 61.41	& 63.44 	& 65.18 \\
								& ORG 		& 41.81 & 47.77 & 41.39 	& 41.33 		& 32.38 	& 31.51	& 36.78	& 35.77 	& 36.62 \\
								& Overall 	& 61.48	& 67.15 & 65.60 	& 65.32 		& 37.90		& 53.63 & 57.16	& 57.54 	& 59.36 \\
\hline
\end{tabular}}
\caption{Model performance by F1 on the testing set of each dataset. The first group of models are all fully-supervised, which use manual fine-grained annotations. while the second group of models use only named entity dictionaries to perform the NER task.}
\label{table:overall_result}
\end{table*}

\begin{figure*}[t!]
\centering
\subfloat[PER]      {
\includegraphics[width=.32\linewidth]{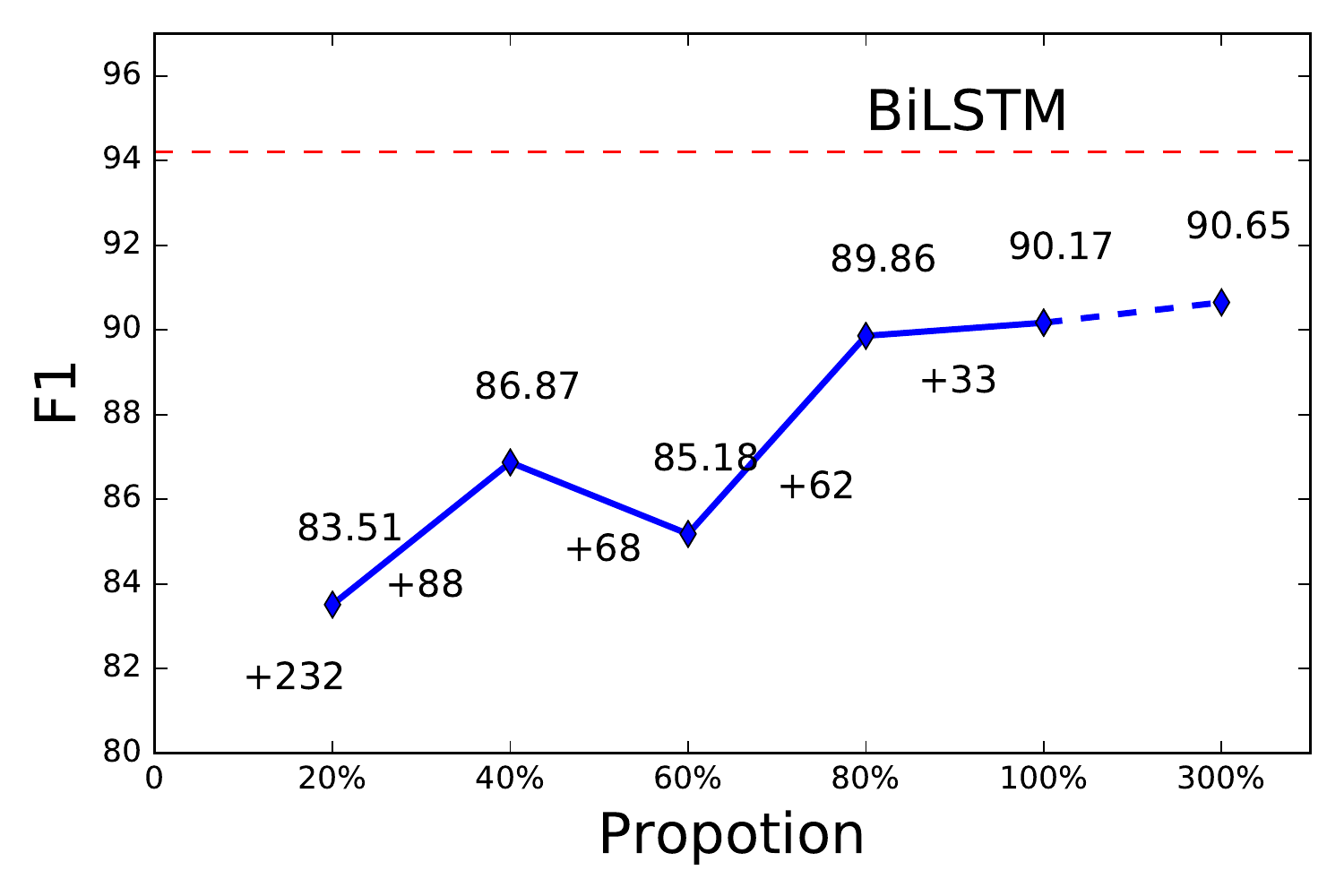}
}
\subfloat[LOC]      {
\includegraphics[width=.32\linewidth]{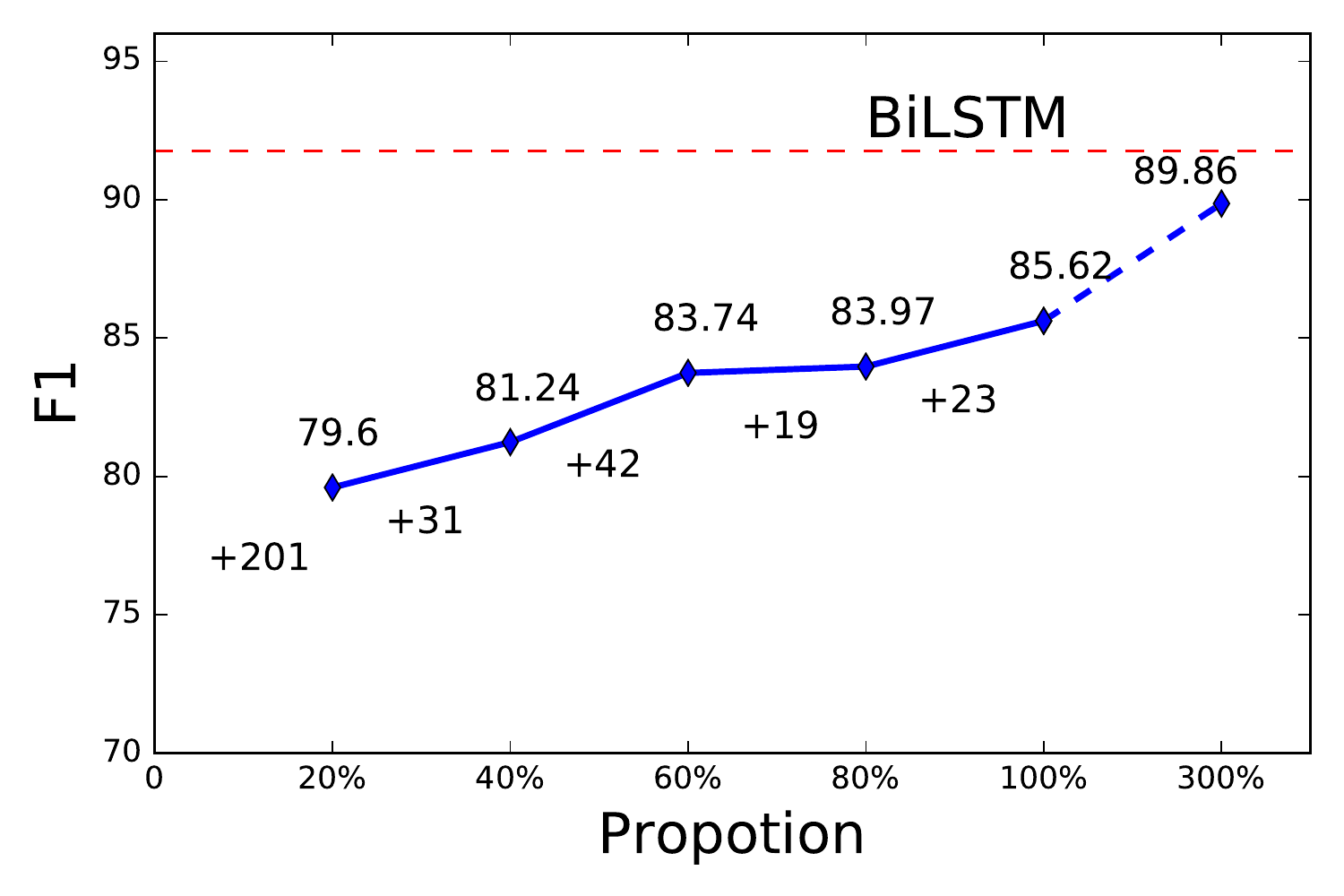}
}
\subfloat[ORG]      {
\includegraphics[width=.32\linewidth]{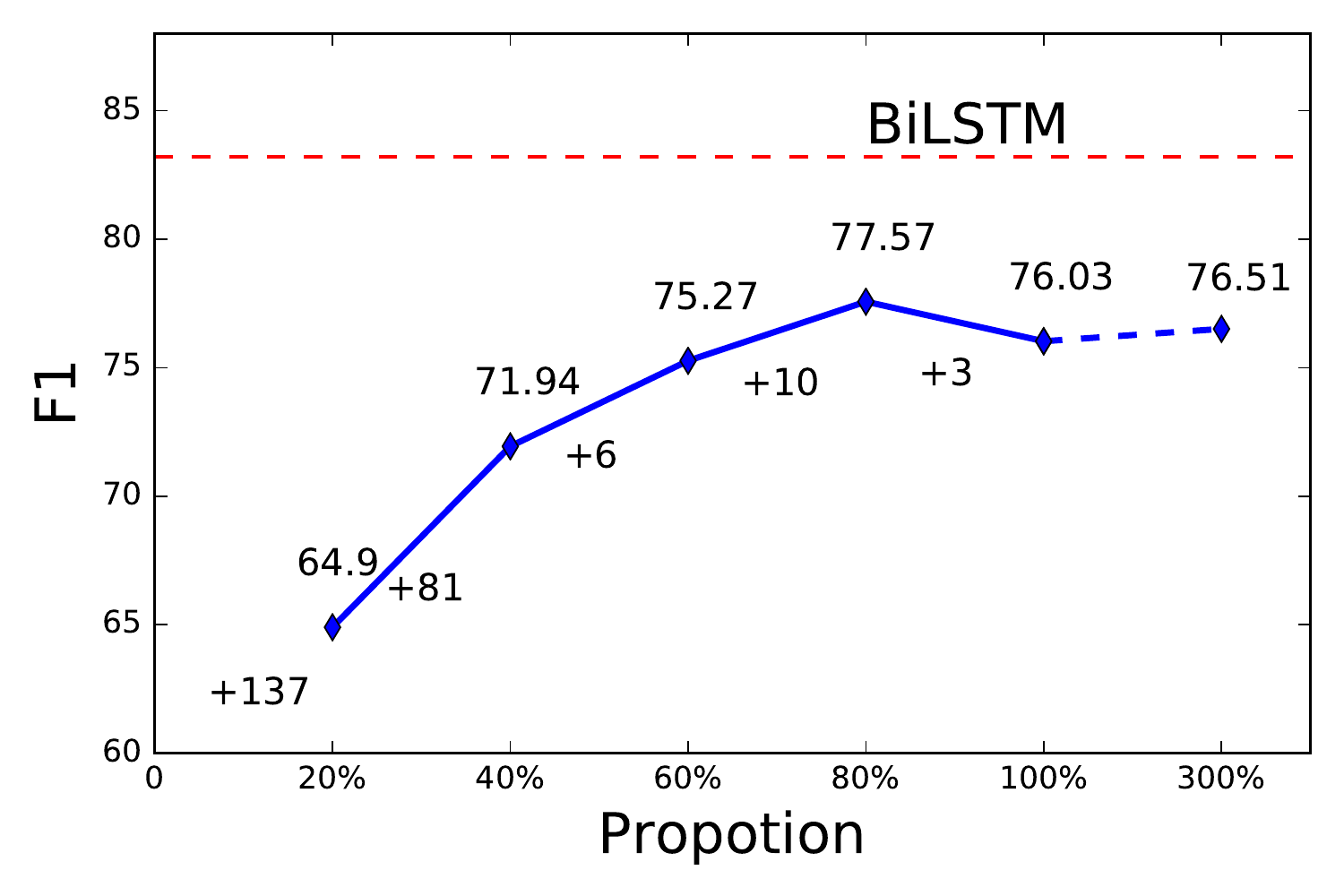}
}
\caption{F1 of AdaPU on the testing set of CoNLL (en) using different portion of the training data set for model training. The red dot line denotes performance of BiLSTM. '+$k$' means that it labels $k$ more unique words on the additional 20\% (e.g., 40\%-20\%) of training data. %Note that for our proposed method, we do not use the gold labels of the training set.
}
\label{fig:data_propotion}
\end{figure*}

\subsection{Build Named Entity Dictionary}
%As we have mentioned, our proposed method can generalize well to a quite simple dictionary.
%for the proposed method, it does not require the dictionary to cover many entities and entries of the dictionary can be incomplete entities. Therefore, the dictionary can be quite simple.
%In the following, we illustrate the detail for constructing the dictionary of each entity type.

For CoNLL (en), MUC, and Twitter datasets, we collected the first 2,000 popular English names in England and Wales in 2015 from ONS\footnote{http://www.ons.gov.uk/ons/index.html} to construct the PER dictionary. 
For LOC, we collected names of countries and their top two popular cities\footnote{https://en.wikipedia.org/wiki/List\_of\_countries\_by national\_capital\_largest\_and\_second-largest\_cities} to construct the dictionary. While for MISC, we turned country names into the adjective forms, for example, England $\rightarrow$ English, and China $\rightarrow$ Chinese, and used the resultant forms to construct the dictionary.
For ORG, we collected names of popular organizations and their corresponding abbreviations from Wikipedia \footnote{https://en.wikipedia.org/wiki/List\_of\_intergovernmental\_ organizations} to construct the dictionary. We also added names of some international companies\footnote{https://en.wikipedia.org/wiki/List\_of\_multinational\_corp-orations}, such as Microsoft, Google, and Facebook, into the dictionary. In addition, we added some common words occurring in organization names such as ``Conference", ``Cooperation", ``Commission", and so on, into the dictionary.

For CoNLL (sp), we used DBpedia query editor\footnote{http://dbpedia.org} to select the most common 2000 names of the people who was born in Spain to construct the PER dictionary. We further used Google translator to translate the English LOC, ORG, MISC dictionary into Spanish.

The resultant named entity dictionaries contain 2,000 person names, 748 location names, 353 organization names, and 104 MISC entities. Table \ref{table:dict-info} lists some statistic information of the data labeling results with these dictionaries using Alg. \ref{alg:entity_label}. From the table, we can see that the precision of the data labeling is acceptable but the recall is quite poor. This is expectable and is a typical problem of the method using only dictionaries to perform NER.

\subsection{Estimate $\pi_p$}
Before disscussing the estimation of $\pi_p$ defined in Eq. (\ref{eq:est_pi}), let us first look at some statistic information of the four studied datasets. Table \ref{table:data_prior} lists the true value of $\pi_p=(\text{\# of entity words}) / (\text{\# of words of the training set})$ for different entity types over dataset. From the table, we can see that the variation of $\pi_p$ cross different datasets is quite small. This motivates us to use the value of $\pi_p$ obtained from an existing labeled dataset as an initialization. The labeled dataset may be from other domains or be out-of-date. In this work, we initially set $\pi_p=0.04, 0.04, 0.05, 0.03$ for PER, LOC, ORG, and MISC, respectively. Starting from this value, we trained the proposed model and used it to perform prediction on the unlabeled dataset. Based on the predicted results, we re-estimate the value of $\pi_p$. The resulted values are listed in table \ref{table:data_prior} and were used throughout our experiments without further illustration.

%\subsection{Implementation Detail}
%We implemented the \textit{one-layer convolution network} with 30 kernels. Window size of each kernel was set to 3. Hidden size of the forward and backward LSTM was set to 200, respectively. Dropout was applied to the word embeddings $E_w(\bm{w})$ with a dropping probability of 0.5. Parameter updating was implemented using the Adam \cite{kinga2015method} optimizer with learning rate set to be 1e-4. 

\subsection{Results} 

Following the protocol of most previous works, we apply the entity-level (exact entity match) F1 to evaluate model performance. 

\paragraph{General Performance.} Table \ref{table:overall_result} shows model performance by entity type and the overall performance on the four tested datasets.
From the table, we can observe: \textbf{1)} The performance of the Matching model is quite poor compared to other models. We found out that it mainly resulted from low recall values. This accords with our discussion in \textsection \ref{sec:introduction} and shows its inapplicability using such simple dictionaries. \textbf{2)} Those PU-learning-based methods achieve significant improvement over Matching on all datasets. This demonstrates the effectiveness of the PU learning framework for NER in the studied setting. \textbf{3)} buPU greatly outperforms uPU. This verifies our analysis in \textsection  \ref{sec:consistent_pu} about the necessity to make $\ell$ upper bounded. \textbf{4)} bnPU slightly outperforms buPU on most of datasets and entity types. This verifies the effectiveness of the non-negative constraint proposed by \citet{kiryo2017positive}. \textbf{5)} The proposed AdaPU model achieves further improvement over bnPU, and it even achieves comparable results with some supervised methods, especially for the PER type. This verifies the effectiveness of our proposed method for enriching the named entity dictionaries.  %This offers a positive answer to our first question in the beginning of the ``Experiments" section.

\begin{table}[h]
\renewcommand{\arraystretch}{0.8}
\centering
\resizebox{\columnwidth}{!}
{
\begin{tabular}{llccc}
\toprule
Type & \multicolumn{1}{c}{Size} & Precision & Recall  \\ \midrule
PER &10,159 (2,000) &89.65 (89.26)  &19.08 (17.38) \\
LOC &10,106 (748) &71.77 (85.07)  & 56.42 (50.03)\\ 
ORG &10,039 (353) &83.42 (86.17) &28.59 (29.45) \\ 
\bottomrule
\end{tabular} }
\caption{Statistic information of the extended dictionary v.s. (that of the original dictionary).}
\label{table:dict-info2}
\end{table}

\begin{table}[h]
\centering
\resizebox{\columnwidth}{!}{\begin{tabular}{lcccc}
\toprule
Model & PER & LOC & ORG & Overall\\ \midrule
Matching & 9.10 (6.70) & 69.85 (67.16) &45.52 (46.65) & 41.40 (39.39)\\ 
AdaPU &91.14 (90.17) &77.60 (85.62) &76.67 (76.03) & 81.87 (82.94) \\
\bottomrule
\end{tabular}}
\caption{F1 of the proposed method using the extend dictionary v.s. (that using the original dictionary) on CoNLL (en) testing set.}
\label{table:big_dic_result}
\end{table}

\paragraph{Influence of Unlabeled Data Size.} We further study the influence of the unlabeled data size to our proposed method. To perform the study, we used 20\%, 40\%, 60\%, 80\%, 100\%, and 300\% (using additional unlabeled data) of the training data set of CoNLL (en) to train AdaPU, respectively. Figure \ref{fig:data_propotion} depicts the results of this study on PER, LOC, and ORG. From the figure, we can see that increasing the size of training data will, in general, improve the performance of AdaPU, but the improvements are diminishing. Our explanation of this phenomenon is that when the data size exceeds a threshold, the number of unique patterns becomes an sublinear function of the data size. This was verified by the observation from the figure, for example, on PER, it labeled 232 unique words on 20\% of training data, while it only labeled 88 more unique words after introducing additional 20\% of training data.

\paragraph{Influence of Dictionary.} We then study the influence of the dictionary on our proposed model. To this end, we extended the dictionary with DBpedia using the same protocol proposed by \citet{chiu2016named}. Statistic information of the resultant dictionary is listed in table \ref{table:dict-info2}, and model performance using this dictionary is listed in table \ref{table:big_dic_result}. A noteworthy observation of the results is that, on LOC, the performance should decrease a lot when using the extended dictionary. We turn to table \ref{table:dict-info2} for the explanation. We can see from the table that, on LOC, the data labeling precision dropped about 13 points (85.07 $\rightarrow$ 71.77) using the extend dictionary. This means that it introduced more false-positive examples into the PU learning and made the empirical risk estimation bias more to the expectation when using the extended dictionary. %This also indicates that the proposed method does not require the dictionary to cover many entities, but entries of the dictionary should be of entities with high confidence. 

\begin{table}[t!]
\centering
\resizebox{\columnwidth}{!}
{
\begin{tabular}{lccccc}
\toprule
$\pi_p$ & PER & LOC & ORG & MISC  & Overall \\ \midrule
True & 90.21  & 85.06  & 77.17 &69.85  & 83.13 \\ 
Estimated & 90.17 & 85.62  & 76.03 &69.30 & 82.94\\ 
\bottomrule
\end{tabular}}
\caption{F1 of the proposed method on CoNLL (en) when using True/Estimated value of $\pi_p$.}
\label{table:pi_value}
\end{table}

\paragraph{Influence of $\pi_p$ Value.} Table \ref{table:pi_value} lists the performance of AdaPU when using the true or estimated value of $\pi_p$ as listed in table \ref{table:data_prior}. From the table, we can see that the proposed model using the estimated $\pi_p$ only slightly underperforms that using the true value of $\pi_p$. %Moreover, on LOC, the model using the estimated $\pi_p$ even performs somehow better that that using the true value of $\pi_p$. 
This shows the robustness of the proposed model to a small variation of $\pi_p$ and verifies the effectiveness of the $\pi_p$ estimation method.

\section{Related Work}
%Once a time, NER methods have heavily relied on dictionaries. Most of them use dictionary matching method to perform the task. A common feature of these methods is that they heavily rely on the quality of the dictionary. Only in this way, they can apply their disambiguitious rules \cite{nadeau2006unsupervised,gerner2010linnaeus,liu2015effects} or self-training \cite{ghiasvand2017unsupervised} to address the entity-noun error, boundary detection error, and out of the dictionary error. 

%\subsection{Semi-supervised Learning}

Positive-unlabeled (PU) learning \cite{li2005learning} aims to train a classifier using only labeled positive examples and a set of unlabeled data, which contains both positive and negative examples. Recently, PU learning has been used in many applications, e.g., text classification \cite{li2003learning}, matrix completion \cite{hsieh2015pu}, and sequential data \cite{nguyen2011positive}. The main difference between PU learning and semi-supervised learning is that, in semi-supervised learning, there is labeled data from all classes, while in PU learning, labeled data only contains examples of a single class .

%To address the negative empirical risks in PU learning, non-negative PU learning (nnPU) \cite{kiryo2017positive} has been proposed. In this paper, we design as new algorithm to do NER task in the PU learning framework. By using dictionaries and unlabeled data, we can solve NER task, which greatly reduces the cost of our labelling data.

%Semi-supervised learning such as self-training \cite{kozareva2005self}, co-training \cite{blum1998combining}, and bootstrap \cite{efron1992bootstrap}, is a class of machine learning algorithms that utilizing unlabeled data to enhance generalization of supervised models. The main difference between semi-supervised learning and PU learning is that, in semi-supervised learning, there is labeled data from all classes, while in PU learning, labeled data only contains examples of one class.

%\subsection{AdaSampling}
AdaSampling \cite{yang2017positive} is a self-training-based approach designed for PU learning, which utilizes predictions of the model to iteratively update training data. 
Generally speaking, it initially treats all unlabeled instances as negative examples.
%, which are of probability 1 to be selected for model training. 
Then, based on the model trained in the last iteration, it generates the probability $p(y=0|\bm{x}_i^u)$ of an unlabeled example $\bm{x}_i^u$ to be a negative one. This value, in turn, determines the probability of $\bm{x}_i^u$ to be selected as the negative examples for model training in next iteration. This process iterates for an acceptable result.

\section{Conclusion}

In this work, we introduce a novel PU learning algorithm to perform the NER task using only unlabeled data and named entity dictionaries. We prove that this algorithm can unbiasedly and consistently estimate the task loss as if there is fully labeled data. And we argue that it can greatly reduce the requirement on sizes of the dictionaries. Extensive experimental studies on four NER datasets validate its effectiveness. 
% and show that the proposed method can even achieve comparable results with state-of-the-art supervised methods, using quite simple dictionaries.
%This work proposes a novel method to perform NER. Different from previous supervised learning approaches that need word-level labeling, the proposed method only needs an unlabeled dataset and a named entity dictionary for each entity type to perform the task. This makes it easy to be applied to new domains, where labeled data is hard to obtain. We empirically evaluated the proposed method on three benchmark NER datasets. 

\section*{Acknowledgements}
The authors wish to thank the anonymous reviewers for their helpful comments. This work was partially funded by China National Key R\&D Program (No. 2018YFB1005104, 2018YFC0831105, 2017YFB1002104, ), National Natural Science Foundation of China (No.  61751201, 61532011), Shanghai Municipal Science and Technology Major Project (No.2018SHZDZX01), STCSM (No.16JC1420401,17JC1420200), ZJLab.

\bibliographystyle{acl_natbib}
\bibliography{PU_learning_for_NER}

\appendix

\section{Proof of Theorem 1}

\begin{proof}
Let denote $\hat{R}^s_{\ell}(f)$ the empirical estimation of $R_{\ell}(f)$ with $k$ randomly labeled examples. Since $\ell$ is bounded, $C_0$, $M$, and $B$ are finite. According to the Lemma in \cite{rosasco2004loss} we have:
\begin{equation} \label{eq:lemma}
\begin{split}
        &\rm{P}\{S \in \mathcal{D}|\sup_{f \in \mathcal{H}_R} |{R}_{\ell}(f)-\hat{R}^s_{\ell}(f)| \leq \epsilon\} \\
        &\geq 1 - 2N(\frac{\epsilon}{4L_M}) e^{-\frac{k\epsilon^2}{8B^2}}.
\end{split}   
\end{equation}
Then, the empirical estimation error of ${R}_{\ell}(f) - \hat{R}_{\ell}(f)$ in PU learning can be written as:
\begin{equation}
\begin{split}
&R_{\ell}(f) - \hat{R}_{\ell}(f) \\
&= \left(\mathbb{E}_{\rbm{X}}\ell(f(\bm{x}), 0)-\frac{1}{n_u} \sum_{i=1}^{n_u} \ell((f(x_i^u), 0)\right) \\
& +  {\pi_p} \left( \mathbb{E}_{\rbm{X}|\rm{Y}=1} \ell(f(\bm{x}), 1)-\frac{1}{n_p}\sum_{i=1}^{n_p} \ell(f(x_i^p), 1) \right)\\
& - {\pi_p} \left( \mathbb{E}_{\rbm{X}|\rm{Y}=1} \ell(f(\bm{x}), 0)-\frac{1}{n_p}\sum_{i=1}^{n_p} \ell(f(x_i^p), 0) \right)
\end{split}
\end{equation}
Thus,
\begin{equation}
\begin{split}
&|R_{\ell}(f) - \hat{R}_{\ell}(f)| \\
&\leq \left|\mathbb{E}_{\rbm{X}}\ell(f(\bm{x}), 0)-\frac{1}{n_u} \sum_{i=1}^{n_u} \ell((f(x_i^u), 0)\right| \\
& +  {\pi_p} \left| \mathbb{E}_{\rbm{X}|\rm{Y}=1} \ell(f(\bm{x}), 1)-\frac{1}{n_p}\sum_{i=1}^{n_p} \ell(f(x_i^p), 1) \right|\\
& + {\pi_p} \left| \mathbb{E}_{\rbm{X}|\rm{Y}=1} \ell(f(\bm{x}), 0)-\frac{1}{n_p}\sum_{i=1}^{n_p} \ell(f(x_i^p), 0) \right|
\end{split}
\end{equation}
Let $\rm{I}_{\ell}(\rbm{X}, 0)$ denote $$\mathbb{E}_{\rbm{X}}\ell(f(\bm{x}), 0)-\frac{1}{n_u} \sum_{i=1}^{n_u} \ell((f(x_i^u), 0).$$ According to Eq. \ref{eq:lemma}, we have:
\begin{equation}
\begin{split}
&\rm{P}\{S \in \mathcal{D}|\sup_{f \in \mathcal{H}_R} |\rm{I}_{\ell}(\rbm{X}, 0)| \leq \epsilon\} \\
& \geq 1-2N(\frac{\epsilon}{4L_M}) e^{-\frac{n_u\epsilon^2}{8B^2}}
\end{split}
\end{equation}
Similarly, let $\rm{I}_{\ell}(\rbm{X}|\rm{Y}=1, 1)$ denote $$\mathbb{E}_{\rbm{X}|\rm{Y}=1} \ell(f(\bm{x}), 1)-\frac{1}{n_p}\sum_{i=1}^{n_p} \ell(f(x_i^p), 1),$$ and $\rm{I}_{\ell}(\rbm{X}|\rm{Y}=1, 0)$ denote $$\mathbb{E}_{\rbm{X}|\rm{Y}=1} \ell(f(\bm{x}), 0)-\frac{1}{n_p}\sum_{i=1}^{n_p} \ell(f(x_i^p), 0),$$
we have:
\begin{equation}
\begin{split}
&\rm{P}\{S \in \mathcal{D}|\sup_{f \in \mathcal{H}_R} |\rm{I}_{\ell}(\rbm{X}|\rm{Y}=1, 1)| \leq \epsilon\} \\
& \geq 1-2N(\frac{\epsilon}{4L_M}) e^{-\frac{n_p\epsilon^2}{8B^2}},
\end{split}
\end{equation}
and 
\begin{equation}
\begin{split}
&\rm{P}\{S \in \mathcal{D}|\sup_{f \in \mathcal{H}_R} |\rm{I}_{\ell}(\rbm{X}|\rm{Y}=1, 0)| \leq \epsilon\} \\
& \geq 1-2N(\frac{\epsilon}{4L_M}) e^{-\frac{n_p\epsilon^2}{8B^2}},
\end{split}
\end{equation}
Therefore, 
\begin{equation}
\begin{split}
&\rm{P}\{S \in \mathcal{D}|\sup_{f \in \mathcal{H}_R} |R_{\ell}(f)-\hat{R}_{\ell}(f)| \leq (1+2\pi_p)\epsilon \} \\
& \geq \min (1-2N(\frac{\epsilon}{4L_M}) e^{-\frac{n_p\epsilon^2}{8B^2}}, \\
& \quad \quad 1-2N(\frac{\epsilon}{4L_M}) e^{-\frac{n_u\epsilon^2}{8B^2}}) \\
&= 1-2N(\frac{\epsilon}{4L_M}) e^{-\frac{\min(n_p, n_u)\epsilon^2}{8B^2}}
\end{split}
\end{equation}
The theorem follows replacing $\epsilon$ with $\frac{1}{1+2\pi_p}\epsilon$.
\end{proof}

\end{document}